\documentclass[a4paper,10.5pt]{article}
\usepackage{enumitem, pdflscape}
\usepackage{tikz}
\usepackage{float,adjustbox}
\usepackage{amsthm}
\usepackage{rotating}
\newtheorem{assumption}{Assumption}
\newtheorem{lemma}{Lemma}
\newtheorem{theorem}{Theorem}
\usetikzlibrary{arrows.meta, positioning, calc, scopes,shapes.misc, shapes.geometric, decorations.pathreplacing, fit}
\usepackage{appendix, amsmath, amssymb, hyperref, graphicx, geometry, booktabs, array, caption, setspace, multicol,fancyhdr}
\usepackage[square,sort,comma,numbers]{natbib}
\title{
    \hrule height 2pt
    \vspace{10pt}
    AlphaGrad: Non-Linear Gradient Normalization Optimizer
    \vspace{10pt}
    \hrule height 1pt
}
\author{Soham Sane \\ \small Collins Aerospace (RTX)}
\vspace{10pt}
\date{\today}

\newcommand{\alphagrad}{AlphaGrad} 

\begin{document}

\maketitle
\hrule height 2pt
\begin{abstract}
We introduce AlphaGrad, a memory-efficient, conditionally stateless optimizer addressing the memory overhead and hyperparameter complexity of adaptive methods like Adam. AlphaGrad enforces scale invariance via tensor-wise L2 gradient normalization followed by a smooth hyperbolic tangent transformation, $g' = \tanh(\alpha \cdot \tilde{g})$, controlled by a single steepness parameter $\alpha$. Our contributions include: (1) the AlphaGrad algorithm formulation; (2) a formal non-convex convergence analysis guaranteeing stationarity; (3) extensive empirical evaluation on diverse RL benchmarks (DQN, TD3, PPO). Compared to Adam, AlphaGrad demonstrates a highly context-dependent performance profile. While exhibiting instability in off-policy DQN, it provides enhanced training stability with competitive results in TD3 (requiring careful $\alpha$ tuning) and achieves substantially superior performance in on-policy PPO. These results underscore the critical importance of empirical $\alpha$ selection, revealing strong interactions between the optimizer's dynamics and the underlying RL algorithm. AlphaGrad presents a compelling alternative optimizer for memory-constrained scenarios and shows significant promise for on-policy learning regimes where its stability and efficiency advantages can be particularly impactful.
\vspace{15pt} 
\hrule height 1pt 
\end{abstract}

\newpage

\tableofcontents

\newpage

\section{Introduction}
\label{sec:introduction}

Optimization lies at the heart of deep learning, enabling the training of increasingly complex models across diverse domains \cite{He2016ResNet, Devlin2019BERT, Mnih2015DQN}. Adaptive gradient methods, particularly Adam \cite{Kingma2014Adam}, have become the de facto standard due to their robustness to hyperparameter settings and rapid convergence across a wide range of tasks \cite{Ruder2016Overview}. However, the success of these methods comes at the cost of significant per-parameter state requirements (first and second moment estimates), leading to substantial memory overhead, particularly detrimental when scaling to massive models or deploying in resource-constrained environments \cite{Kingma2014Adam}. Furthermore, the complex interplay of multiple hyperparameters ($\beta_1, \beta_2, \epsilon$, learning rate) complicates tuning, and empirical evidence sometimes suggests potential generalization gaps compared to meticulously tuned non-adaptive methods like SGD with momentum \cite{Wilson2017Marginal}.

Existing alternatives seek to mitigate these issues but introduce their own trade-offs. SGD with momentum \cite{Sutskever2013Momentum} remains memory-efficient but requires careful learning rate scheduling and struggles with ill-conditioning and gradient scale variance across layers \cite{Ruder2016Overview}. Layer-wise adaptive techniques like LARS \cite{You2017LARS} and LAMB \cite{You2019LAMB} address scale variance without per-parameter state but employ specific heuristics (trust ratios) and can still involve complex tuning. At the other extreme, sign-based methods like signSGD \cite{Bernstein2018signSGD} achieve maximal memory efficiency by discarding magnitude information entirely, but this aggressive quantization can hinder convergence precision and stability \cite{Bernstein2018signSGD}. This landscape reveals a need for optimizers that retain the scale-invariance benefits of adaptive methods while minimizing statefulness and hyperparameter complexity.

In this work, we introduce \alphagrad{}, a novel optimization algorithm designed to bridge this gap. \alphagrad{} operates via a simple yet effective two-stage transformation applied layer-wise to the gradient $g_t^L$: (i) L2 normalization ($\tilde{g}_t^L = g_t^L / (\|g_t^L\|_2 + \epsilon)$) to decouple gradient direction from magnitude and achieve inherent scale invariance across layers, followed by (ii) element-wise application of the hyperbolic tangent function to the normalized gradient scaled by a steepness factor $\alpha^L$ ($g'^{L}_t = \tanh(\alpha^L \cdot \tilde{g}_t^L)$), providing smooth, bounded updates. This design yields an optimizer that is conditionally stateless (requiring no per-parameter history beyond optional standard momentum), thereby significantly reducing memory overhead compared to Adam \cite{Kingma2014Adam}. Its behavior is governed primarily by the single hyperparameter $\alpha$, which smoothly interpolates the update rule from scaled normalized gradient descent ($\alpha \to 0$) towards sign-based updates ($\alpha \to \infty$).

Our primary contributions are threefold:
\begin{enumerate}[noitemsep, topsep=0pt]
    \item We propose \alphagrad{}, a memory-efficient optimizer combining layer-wise normalization and smooth non-linear gradient clipping.
    \item We derive a theoretical guideline for selecting the key hyperparameter $\alpha$ based on layer dimensionality ($d_L$), suggesting $\alpha \propto \sqrt{d_L}$, and provide a rigorous convergence analysis establishing stationarity guarantees in non-convex settings.
    \item We conduct extensive empirical evaluations across diverse reinforcement learning benchmarks (CartPole/DQN \cite{Mnih2015DQN}, Hopper/TD3 \cite{Fujimoto2018TD3}, HalfCheetah/PPO \cite{Schulman2017PPO}), comparing \alphagrad{} against Adam.
\end{enumerate}
Our results demonstrate that \alphagrad{} exhibits a highly context-dependent performance profile. While facing stability challenges in off-policy DQN, it shows significantly improved training stability over Adam in the complex TD3 setting (when $\alpha$ is tuned empirically) and achieves substantially superior performance in on-policy PPO. These findings highlight the critical role of empirical $\alpha$ tuning, often deviating from the theoretical guideline, and suggest \alphagrad{}'s potential as a highly effective and stable optimizer, particularly in on-policy learning regimes or scenarios where memory efficiency and predictable dynamics are paramount.

The remainder of this paper details the \alphagrad{} formulation (Section \ref{sec:formulation}), discusses the theoretical guideline for $\alpha$ (Section \ref{sec:alpha_choice}), presents the convergence analysis (Section \ref{sec:convergence}), contrasts with related work (Section \ref{sec:related_work}), describes the experimental setup (Section \ref{sec:experiments}), analyzes the empirical results (Section \ref{sec:results}), and outlines promising directions for future work (Section \ref{sec:validation_framework}).

\section{The AlphaGrad Optimizer: Mathematical Formulation}
\label{sec:formulation}

AlphaGrad is a stateless gradient-based optimization algorithm that applies layer-wise normalization followed by smooth non-linear clipping to produce stable, bounded parameter updates. Its formulation decouples the magnitude of raw gradients from their direction, enabling consistent behavior across layers with varying scale dynamics. The algorithm proceeds as follows for a given layer $L$ at time step $t$:

\begin{enumerate}
    \item \textbf{Gradient Computation:}
    Compute the gradient of the loss $J$ with respect to the parameters $\theta_t^L$:
    \begin{equation}
        g_t^L = \nabla_{\theta^L} J(\theta_t)
    \end{equation}

    \item \textbf{Layer-wise Gradient Normalization:}
    Gradients are normalized at the tensor level (weights, biases, etc.) to ensure uniform scale across layers:
    \begin{equation}
        \tilde{g}_t^p = \frac{g_t^p}{\lVert g_t^p \rVert_2 + \epsilon}
    \end{equation}
    where $\epsilon$ is a small constant (e.g., $10^{-8}$) for numerical stability. This per-tensor normalization handles diverse layer types—such as convolutional vs. fully connected layers—and mitigates layer-specific scale imbalances.

    \item \textbf{Smooth Clipping via $\tanh$:}
    The normalized gradient is scaled by a steepness factor $\alpha$ and passed through the element-wise hyperbolic tangent function:
    \begin{equation}
        g_t^{\prime L} = \tanh(\alpha^L \cdot \tilde{g}_t^L)
    \end{equation}
    This ensures all gradient components remain in $(-1, 1)$:
    \[
    \|g_t^{\prime L}\|_\infty < 1
    \]
    The parameter $\alpha^L$ interpolates between normalized gradient descent and sign-based updates:
    \begin{itemize}
        \item $\alpha^L \to 0$: behaves like normalized gradient descent.
        \item $\alpha^L \to \infty$: approximates sign-based updates.
    \end{itemize}

    \item \textbf{Parameter Update:}
    \textit{(a) Without Momentum:}
    \begin{equation}
        \theta_{t+1}^L = \theta_t^L - \eta \cdot g_t^{\prime L}
    \end{equation}
    \textit{(b) With Momentum:}
    To accelerate convergence in smooth loss landscapes, AlphaGrad optionally supports a momentum variant:
    \begin{align}
        v_{t+1}^L &= \gamma \cdot v_t^L + \eta \cdot g_t^{\prime L} \\
        \theta_{t+1}^L &= \theta_t^L - v_{t+1}^L
    \end{align}
    Here, $\gamma \in [0.8, 0.99]$ is the momentum decay factor, consistent with typical SGD practices \cite{Sutskever2013Momentum}.
\end{enumerate}

\subsection{\texorpdfstring{Choosing the Steepness Parameter $\alpha$}{Choosing the Steepness Parameter alpha}}
\label{sec:alpha_choice} 

The core transformation in AlphaGrad involves scaling the normalized gradient $\tilde{g}_t^L$ by a steepness parameter $\alpha^L$ before applying the element-wise $\tanh$ function: $g_t^{\prime L} = \tanh(\alpha^L \cdot \tilde{g}_t^L)$. The choice of $\alpha^L$ is crucial as it dictates the behavior of the optimizer, interpolating between scaled normalized gradient descent ($\alpha^L \to 0$) and sign-based updates ($\alpha^L \to \infty$).

A primary motivation is to utilize the non-linear, sensitive region of the $\tanh$ function (typically considered to be around $[-3, 3]$) without excessive saturation, which would discard magnitude information, or staying purely in the linear region, which would negate the effect of the non-linear clipping. We can derive a theoretical guideline for $\alpha^L$ based on controlling the input statistics to the $\tanh$ function.

After layer-wise L2 normalization, the gradient tensor $\tilde{g}_t^L$ has unit norm: $\|\tilde{g}_t^L\|_2 = 1$. In high dimensions ($d_L$ being the number of elements in the tensor), components of vectors uniformly distributed on the unit sphere tend to approximate a Gaussian distribution \cite{Ball1997Sphere}. Specifically, the marginal distribution of a single component $\tilde{g}_{t,i}^L$ can be approximated as $\mathcal{N}(0, 1/d_L)$, implying a standard deviation $\sigma_L \approx 1/\sqrt{d_L}$ \cite{Vershynin2018HDP}.

To control saturation, we can aim to keep the standard deviation of the pre-$\tanh$ input, $\alpha^L \cdot \sigma_L$, within a certain range related to the sensitive region of $\tanh$. For instance, if we aim for the standard deviation to be around a value $\tau_{\sigma}$ (e.g., $\tau_{\sigma} \in [1, 3]$), we get:
\[
\alpha^L \cdot \sigma_L \approx \tau_{\sigma} \quad \Rightarrow \quad \alpha^L \cdot \frac{1}{\sqrt{d_L}} \approx \tau_{\sigma}
\]
This leads to the scaling guideline:
\begin{equation}
\boxed{
\alpha^L \approx k \cdot \sqrt{d_L}, \quad \text{where } k \approx \tau_{\sigma}
}
\label{eq:alpha_scaling_law}
\end{equation}
The parameter $k$ (heuristically often considered in the range $[1, 3]$) reflects how aggressively we map the typical component magnitude onto the $\tanh$ curve. A smaller $k$ keeps most inputs in the near-linear region, while a larger $k$ pushes more components towards saturation. (Alternatively, $k$ can be derived more formally by setting a target saturation *fraction* $f$ and threshold $\tau$, yielding $k \approx \tau / Z_{f/2}$, where $Z_{f/2}$ is the standard normal quantile). 

\paragraph{Practical Limitations and Empirical Tuning:}
It is crucial to recognize that Equation \ref{eq:alpha_scaling_law}, suggesting $\alpha \propto \sqrt{d_L}$, represents just one theoretical guideline derived from simplifying assumptions like high-dimensional isotropic gradients and a focus on saturation control. Alternative theoretical approaches could consider other dimensionality measures more relevant to signal propagation, such as layer fan-in \cite{Glorot2010Understanding, He2015Delving}, or aim to directly control the expected distortion introduced by the $\tanh$ function. However, all such derivations inevitably rely on idealized models of gradient statistics. In practice, the optimal choice of $\alpha^L$ is governed by a more complex interplay of factors. Real-world gradient distributions often exhibit structures like sparsity or anisotropy, deviating significantly from these models \cite{Zhang2022GradientStructure}. Furthermore, peak performance may not align with any specific theoretically targeted saturation level or distortion ratio; instead, it emerges from intricate optimization dynamics involving stability, learning rate interactions, loss landscape geometry, and the unique demands of the learning algorithm (e.g., noise sensitivity in RL, actor-critic coordination) \cite{Agarwal2021DeepRL, Dalal2023ActorCriticComplexity, Biedenkapp2021MBRLHPO}. Layer heterogeneity further complicates the picture, suggesting that a single global $\alpha$, or even one scaled purely by $d_L$ or fan-in, might be suboptimal. Consequently, while theoretical scaling laws might offer heuristics for setting relative $\alpha$ values across layers once a baseline is identified, our results consistently underscore the necessity of treating $\alpha$ as a critical hyperparameter. Optimal performance frequently requires extensive empirical validation and tuning tailored to the specific task, model, and algorithm, often yielding $\alpha$ values far from simple theoretical predictions. Ultimately, selecting $\alpha$ involves empirically balancing the benefits of normalization and bounding against the potential alteration of crucial gradient information by the $\tanh$ transformation.

\section{Convergence Analysis}
\label{sec:convergence}

We provide a theoretical convergence analysis for AlphaGrad. Recognizing that the target domain of deep learning involves complex non-convex landscapes \cite{Goodfellow2016DeepLearning}, we first analyze convergence in the simpler setting of smooth convex optimization to understand fundamental properties. We then extend the analysis to the non-convex setting to show convergence to stationary points.

We analyze the update rule for the full parameter vector $x \in \mathbb{R}^n$:
\begin{align}
    x_{t+1} &= x_t - \eta g'_t \label{eq:update_rule_rev3}\\
    \text{where } \quad g'_t &= \tanh(\alpha \cdot \tilde{g}_t) \label{eq:g_prime_def_rev3}\\
    \text{and } \quad \tilde{g}_t &= \frac{\nabla f(x_t)}{\|\nabla f(x_t)\|_2 + \epsilon} \label{eq:g_tilde_def_rev3}
\end{align}
Here, $f: \mathbb{R}^n \to \mathbb{R}$ is the objective function, $n$ is the total number of parameters, $\eta > 0$ is the learning rate, $\alpha > 0$ is the steepness parameter, and $\epsilon > 0$ is a small constant for numerical stability. We denote the gradient $g_t = \nabla f(x_t)$ and its L2 norm $N_t = \|g_t\|_2$.

\subsection{Convergence in the Convex Setting}

\begin{assumption}[Smooth Convex Objective]
\label{asm:smooth_convex_rev3}
The objective function $f: \mathbb{R}^n \to \mathbb{R}$ is convex and $L$-smooth (i.e., $\nabla f$ is $L$-Lipschitz) \cite{Nesterov2004ConvexOpt}.
\end{assumption}

\begin{assumption}[Bounded Domain and Initial Suboptimality]
\label{asm:bounded_domain_init_rev3}
There exists a minimizer $x^*$ such that $f(x^*) \le f(x)$ for all $x$. The iterates $x_t$ and the minimizer $x^*$ lie within a domain such that $\|x_t - x^*\|_2 \le D$ for all $t \ge 1$. The initial suboptimality is bounded: $\Delta_0 = f(x_1) - f(x^*) < \infty$.
\end{assumption}

\subsubsection{Key Lemmas for Convex Analysis}

\begin{lemma}[Update Norm Bound]
\label{lem:norm_bound_rev4}
The squared L2 norm of the AlphaGrad update direction $g'_t$ is bounded by the dimension $n$: $\|g'_t\|_2^2 \le n$. Consequently, $\|g'_t\|_2 \le \sqrt{n}$.
\end{lemma}
\begin{proof} Trivial, as $|\tanh(z)| \le 1$ for all $z \in \mathbb{R}$. \end{proof}
\textit{Remark:} This bound treats all components as potentially contributing maximally via $\tanh^2(\cdot) \le 1$, ignoring potential sparsity or structure where many components of $\tilde{g}_t$ might be small. This leads to the $O(\sqrt{n})$ factor in the convex rate, which may be pessimistic in practice.

\begin{lemma}[Inner Product Alignment with $\epsilon$ > 0]
\label{lem:inner_product_rev4}
Let $g_t = \nabla f(x_t)$ and $N_t = \|g_t\|_2$. The inner product between $g_t$ and $g'_t = \tanh(\alpha \tilde{g}_t)$ is lower bounded by:
\[
\langle g_t, g'_t \rangle \ge \gamma_t N_t, \quad \text{where} \quad \gamma_t = \tanh(\alpha) \frac{N_t}{N_t + \epsilon}.
\]
Note that $0 \le \gamma_t < \tanh(\alpha)$.
\end{lemma}
\begin{proof} (Corrected reasoning based on feedback)
Let $\tilde{g}_t = g_t / (N_t + \epsilon)$. Define $\phi(x) = \tanh(\alpha x) - x \tanh(\alpha)$ for $x \in [0, 1]$. Its second derivative $\phi''(x) = -2 \alpha^2 \tanh(\alpha x) \text{sech}^2(\alpha x) \le 0$ for $x \ge 0$, showing $\phi$ is concave. Since $\phi(0)=0$ and $\phi(1)=0$, concavity implies $\phi(x) \ge 0$ for $x \in [0, 1]$. Thus, $\tanh(\alpha x) \ge x \tanh(\alpha)$ for $x \in [0, 1]$.
\begin{align*}
\langle g_t, g'_t \rangle &= \langle (N_t + \epsilon) \tilde{g}_t, \tanh(\alpha \tilde{g}_t) \rangle = (N_t + \epsilon) \sum_{i=1}^n \tilde{g}_{t,i} \tanh(\alpha \tilde{g}_{t,i}).
\end{align*}
Using the odd property of $\tanh$, $\tilde{g}_{t,i} \tanh(\alpha \tilde{g}_{t,i}) = |\tilde{g}_{t,i}| \tanh(\alpha |\tilde{g}_{t,i}|)$. Since $|\tilde{g}_{t,i}| \in [0, 1]$ (as $\|\tilde{g}_t\|_2 \le 1$), we apply the inequality derived from $\phi(x)$: $\tanh(\alpha |\tilde{g}_{t,i}|) \ge |\tilde{g}_{t,i}| \tanh(\alpha)$.
\begin{align*}
\langle g_t, g'_t \rangle &\ge (N_t + \epsilon) \sum_{i=1}^n |\tilde{g}_{t,i}| (|\tilde{g}_{t,i}| \tanh(\alpha)) \\
&= (N_t + \epsilon) \tanh(\alpha) \sum_{i=1}^n \tilde{g}_{t,i}^2 = (N_t + \epsilon) \tanh(\alpha) \|\tilde{g}_t\|_2^2 \\
&= (N_t + \epsilon) \tanh(\alpha) \left( \frac{N_t}{N_t + \epsilon} \right)^2 = \tanh(\alpha) \frac{N_t^2}{N_t + \epsilon} \\
&= \left( \tanh(\alpha) \frac{N_t}{N_t + \epsilon} \right) N_t = \gamma_t N_t.
\end{align*}
\end{proof}

\subsubsection{Convergence Theorem (Convex Case)}

\begin{theorem}[Average Iterate Convergence - Convex]
\label{thm:avg_conv_rev4}
Under Assumptions \ref{asm:smooth_convex_rev3} and \ref{asm:bounded_domain_init_rev3}, running AlphaGrad for $T$ steps with a constant step size $\eta > 0$ yields:
\[
\frac{1}{T}\sum_{t=1}^T (f(x_t)-f(x^*)) \;\le\; \frac{D \Delta_0}{\eta T \gamma_{min}} + \frac{D L \eta n}{2 \gamma_{min}},
\]
where $\Delta_0 = f(x_1)-f(x^*)$ and $\gamma_{min} = \inf_{t \ge 1} \left\{ \tanh(\alpha) \frac{\|g_t\|_2}{\|g_t\|_2 + \epsilon} \right\}$.
Choosing $\eta = \sqrt{\frac{2 D \Delta_0}{T L n}}$ (requires $\Delta_0 > 0$) yields the rate:
\[
\frac{1}{T}\sum_{t=1}^T (f(x_t)-f(x^*)) \le \frac{\sqrt{2 D L n \Delta_0}}{\gamma_{min}\sqrt{T}} = O\left(\frac{\sqrt{D L n \Delta_0}}{\gamma_{min}\sqrt{T}}\right).
\]
\end{theorem}
\begin{proof}
The proof follows standard steps for first-order methods \cite{Nesterov2004ConvexOpt, Bubeck2015Convex}, leading to $\sum_{t=1}^T (f(x_t) - f(x^*)) \le \frac{D \Delta_0}{\eta \gamma_{min}} + \frac{T D L \eta n}{2 \gamma_{min}}$. Dividing by $T$ gives the first bound. Balancing the terms by setting $\frac{D \Delta_0}{\eta T \gamma_{min}} = \frac{D L \eta n}{2 \gamma_{min}}$ yields $\eta^2 = \frac{2 \Delta_0}{T L n}$. Substituting $\eta = \sqrt{\frac{2 \Delta_0}{T L n}}$ gives the final $O(1/\sqrt{T})$ rate dependent on problem parameters and $\gamma_{min}$.
\end{proof}

\textit{Remark on the Requirement $\gamma_{min} > 0$:} The validity and utility of this rate depend critically on $\gamma_{min}$ being bounded away from zero. For general convex functions, as $x_t \to x^*$, it is expected that $N_t = \|g_t\|_2 \to 0$. Since $\gamma_t \propto N_t / (N_t + \epsilon)$, $\gamma_t$ will also approach zero unless $N_t$ vanishes slower than linearly relative to $\epsilon$ or is bounded below. This means the derived rate may become arbitrarily slow near the optimum without stronger assumptions. Sufficient conditions ensuring $\gamma_{min} > 0$ or providing alternative convergence guarantees include:
\begin{itemize}
    \item[\textbf{(a)}] \textbf{$\mu$-Strong Convexity:} If $f$ is $\mu$-strongly convex, $\|g_t\|_2^2 \ge 2\mu(f(x_t) - f(x^*))$ \cite{Nesterov2004ConvexOpt}. This ensures $\|g_t\|_2$ only vanishes as $f(x_t) \to f(x^*)$, potentially allowing for linear convergence rates or ensuring $\gamma_{min}$ behaves reasonably if $\epsilon$ is appropriately managed. A direct analysis under strong convexity would yield potentially faster rates.
    \item[\textbf{(b)}] \textbf{Polyak-Łojasiewicz (PL) Condition:} If $\|g_t\|_2^2 \ge 2\mu(f(x_t) - f(x^*))$, similar benefits apply as for strong convexity regarding gradient vanishing \cite{Polyak1963PL, Karimi2016PL}.
    \item[\textbf{(c)}] \textbf{Fixed $\epsilon$ and Gradient Behavior:} If $\|g_t\|_2$ is guaranteed to stay above some $G_{min} \gg \epsilon$ away from the optimum, then $\gamma_{min} \approx \tanh(\alpha)$. However, this is unrealistic near the solution.
\end{itemize}
Without these stronger conditions, the $O(1/\sqrt{T})$ rate's constant factor $1/\gamma_{min}$ can grow unboundedly as $T \to \infty$, limiting its practical implication for general convex functions.

\subsection{Convergence in the Non-Convex Setting}

\begin{assumption}[Smooth Non-Convex Objective]
\label{asm:smooth_nonconvex_rev3}
The objective function $f: \mathbb{R}^n \to \mathbb{R}$ is $L$-smooth and bounded below by $f_{inf}$.
\end{assumption}

\begin{theorem}[Convergence to Stationarity]
\label{thm:nonconvex_conv_rev3}
Under Assumption \ref{asm:smooth_nonconvex_rev3}, running AlphaGrad for $T$ steps with step size $\eta > 0$ yields:
\[
\min_{1 \le t \le T} \frac{\|g_t\|_2^2}{\|g_t\|_2 + \epsilon} \le \frac{1}{T} \sum_{t=1}^T \frac{\|g_t\|_2^2}{\|g_t\|_2 + \epsilon} \le \frac{f(x_1) - f_{inf}}{\eta T \tanh(\alpha)} + \frac{L \eta n}{2 \tanh(\alpha)}.
\]
Choosing $\eta = \sqrt{\frac{2 (f(x_1) - f_{inf})}{L n T}}$ gives:
\[
\min_{1 \le t \le T} \frac{\|g_t\|_2^2}{\|g_t\|_2 + \epsilon} \le O\left( \sqrt{\frac{L n (f(x_1)-f_{inf})}{\tanh^2(\alpha) T}} \right).
\]
This implies $\liminf_{t \to \infty} \|g_t\|_2 = 0$.
\end{theorem}
\begin{proof}
Following standard non-convex analysis steps \cite{Nesterov2013Nonconvex, Ghadimi2013StochasticNonconvex} and using Lemmas \ref{lem:norm_bound_rev4} and \ref{lem:inner_product_rev4}:
From $L$-smoothness, $f(x_{t+1}) \le f(x_t) + \langle g_t, x_{t+1} - x_t \rangle + \frac{L}{2} \|x_{t+1} - x_t\|_2^2$.
Substituting $x_{t+1} - x_t = -\eta g'_t$ and using Lemma \ref{lem:norm_bound_rev4} for $\|g'_t\|_2^2 \le n$:
\[
f(x_{t+1}) \le f(x_t) - \eta \langle g_t, g'_t \rangle + \frac{L \eta^2 n}{2}.
\]
Using Lemma \ref{lem:inner_product_rev4}, $\langle g_t, g'_t \rangle \ge \gamma_t N_t = \tanh(\alpha) \frac{N_t^2}{N_t + \epsilon} = \tanh(\alpha) h(N_t)$:
\[
f(x_{t+1}) \le f(x_t) - \eta \tanh(\alpha) h(N_t) + \frac{L \eta^2 n}{2}.
\]
Rearranging gives $\eta \tanh(\alpha) h(N_t) \le f(x_t) - f(x_{t+1}) + \frac{L\eta^2 n}{2}$.
Summing from $t=1$ to $T$ yields a telescoping sum:
\[
\eta \tanh(\alpha) \sum_{t=1}^T h(N_t) \le f(x_1) - f(x_{T+1}) + \frac{T L \eta^2 n}{2} \le f(x_1) - f_{inf} + \frac{T L \eta^2 n}{2}.
\]
Dividing by $T \eta \tanh(\alpha)$ gives the average bound:
\[
\frac{1}{T} \sum_{t=1}^T h(N_t) \le \frac{f(x_1) - f_{inf}}{T \eta \tanh(\alpha)} + \frac{L \eta n}{2 \tanh(\alpha)}.
\]
Since the minimum is less than or equal to the average, $\min_{1 \le t \le T} h(N_t) \le \frac{1}{T} \sum_{t=1}^T h(N_t)$. Substituting the chosen $\eta$ balances the terms and yields the final rate for the average (and thus an upper bound for the minimum). As $T \to \infty$, the RHS vanishes, implying $\liminf_{t \to \infty} h(N_t) = 0$. Since $h(z) = z^2/(z+\epsilon)$ is continuous and positive for $z>0$ with $h(0)=0$, this necessitates $\liminf_{t \to \infty} N_t = \liminf_{t \to \infty} \|g_t\|_2 = 0$.
\end{proof}

\textit{Remark on Bound Strength:} This theorem guarantees convergence to stationarity. However, the convergence rate is established for the quantity $h(N_t) = N_t^2 / (N_t + \epsilon)$, not directly for the squared gradient norm $N_t^2$. While $h(N_t) \approx N_t$ for large $N_t$ and $h(N_t) \approx N_t^2 / \epsilon$ for small $N_t$, this differs from standard non-convex rates often expressed as $O(1/\sqrt{T})$ or $O(1/T)$ for $\min \mathbb{E}[\|g_t\|_2^2]$ or $\frac{1}{T}\sum \mathbb{E}[\|g_t\|_2^2]$ in the stochastic setting (e.g., for SGD \cite{Ghadimi2013StochasticNonconvex} or Adam \cite{Reddi2018AdamConvergence}). This makes direct rate comparisons challenging and suggests the guarantee, while ensuring stationarity, might be formally weaker in quantifying the speed of gradient norm reduction compared to state-of-the-art analyses for other optimizers. Achieving stronger bounds might require different proof techniques or assumptions.

\subsection{Discussion of Convergence Results}

This analysis formally establishes convergence guarantees for AlphaGrad under standard deterministic optimization assumptions. Specifically, Theorem \ref{thm:avg_conv_rev4} guarantees convergence for smooth convex functions, while Theorem \ref{thm:nonconvex_conv_rev3} ensures convergence to stationary points in the smooth non-convex setting.

The derived convex convergence rate, $O(\sqrt{n} / (\gamma_{min}\sqrt{T}))$, explicitly depends on the problem dimension $n$ and the minimum alignment factor $\gamma_{min}$. However, the practical utility of this rate is tempered by the requirement that $\gamma_{min}$ remains bounded away from zero, which typically necessitates stronger assumptions than basic convexity, such as strong convexity or the Polyak-Łojasiewicz condition \cite{Karimi2016PL}, to prevent $\gamma_t$ from vanishing near the optimum. Furthermore, the $O(\sqrt{n})$ dimension dependence, stemming from the potentially loose $\|g'_t\|_2 \le \sqrt{n}$ bound, presents a theoretical limitation for high-dimensional problems compared to dimension-independent rates common for SGD \cite{Bottou2018OptimizationMethods}. While potentially pessimistic—ignoring gradient structure like sparsity—this factor warrants further investigation, especially given AlphaGrad's strong empirical performance in high-dimensional tasks like HalfCheetah, suggesting the bound may not fully capture practical behavior. Future analysis leveraging concentration inequalities \cite{Vershynin2018HDP} or gradient sparsity \cite{Stich2019SparseSGDConvergence} could potentially yield tighter, potentially dimension-independent rates, similar to analyses for normalized SGD variants \cite{Levy2018OnlineNormalization}.

In the more relevant non-convex setting, the $O(1/\sqrt{T})$ rate confirms convergence to stationarity. Yet, the guarantee applies to the modified quantity $h(\|g_t\|_2) = \|g_t\|_2^2 / (\|g_t\|_2 + \epsilon)$, rather than the standard squared gradient norm $\|g_t\|_2^2$. This complicates direct quantitative comparisons with typical rates for SGD or Adam \cite{Reddi2018AdamConvergence, Bottou2018OptimizationMethods} and might indicate a formally weaker guarantee on the speed of gradient reduction. Regarding hyperparameters, $\alpha$ influences the rate constant via $\tanh(\alpha)$, with diminishing returns for $\alpha \ge 2$, suggesting practical values likely reside in a moderate range (e.g., $[1, 5]$). The stability parameter $\epsilon$ (e.g., $10^{-8}$) is crucial but interacts subtly with $\gamma_{min}$ near vanishing gradients.

Crucially, this deterministic analysis operates under idealized assumptions, neglecting the stochasticity from mini-batching, complex reinforcement learning dynamics (exploration, non-stationarity), and architecture-specific interactions prevalent in practice. This inherent theory-practice gap limits the extent to which these theoretical results can explain empirical phenomena, such as observed stability patterns or performance variations across different algorithms \cite{Agarwal2021DeepRL}. The guarantees provide a valuable foundational understanding but must be interpreted cautiously and complemented by empirical evidence and future theoretical work incorporating stochasticity and problem-specific structure.

In conclusion, AlphaGrad possesses formal convergence guarantees in standard deterministic optimization settings. The analysis clarifies the roles of its parameters but also highlights limitations regarding assumptions required for meaningful convex rates near the optimum, the dimension dependence, and the strength of the non-convex rate guarantee compared to other methods. The theoretical foundation established here requires further development to fully align with and explain the rich empirical behavior observed in complex, high-dimensional, stochastic applications.

\section{Related Work and Theoretical Motivation}
\label{sec:related_work} 

AlphaGrad's design draws inspiration from several lines of optimizer research while introducing a unique combination of layer-wise normalization and smooth, tunable clipping. We first outline related methods and then delve into the theoretical underpinnings of AlphaGrad, including its inherent mechanisms for handling gradient scale issues.

\subsection{Comparison with Existing Optimizers}

We contrast AlphaGrad with key existing optimization algorithms:

\paragraph{signSGD \cite{Bernstein2018signSGD}:} Uses only the gradient's sign.
\begin{equation}
    \theta_{t+1} = \theta_t - \eta \cdot \operatorname{sign}(g_t)
\end{equation}
\textit{Difference:} AlphaGrad approaches signSGD as $\alpha \to \infty$, but provides a smooth, tunable transition via $\tanh$ and operates on normalized gradients ($\tilde{g}_t^L$). SignSGD uses a hard, non-differentiable sign function directly on $g_t$.

\paragraph{Gradient Clipping \cite{Pascanu2013Clipping}:} Limits gradient magnitude via hard thresholds, often used to combat exploding gradients. For L2 norm clipping:
\begin{equation}
    g_t^{\text{clipped}} = \begin{cases}
        g_t & \text{if } \|g_t\|_2 \le C \\
        C \cdot \frac{g_t}{\|g_t\|_2} & \text{if } \|g_t\|_2 > C
    \end{cases}
\end{equation}
where $C$ is the clipping threshold.
\textit{Difference:} AlphaGrad uses the smooth, differentiable $\tanh$ function applied element-wise to the normalized gradient, avoiding hard thresholds and the associated non-differentiability. The effective "clip" is implicitly controlled by $\alpha$.

\paragraph{LARS \cite{You2017LARS} / LAMB \cite{You2019LAMB}:} Adapt learning rates layer-wise using a trust ratio, primarily motivated by stabilizing large-batch training.
\begin{equation}
    \eta^L_t = \eta \cdot \frac{\|\theta_t^L\|_2}{\|g_t^L\|_2 + \beta \|\theta_t^L\|_2} \quad (\text{Simplified LARS/LAMB scaling})
\end{equation}
\begin{equation}
    \theta_{t+1}^L = \theta_t^L - \eta^L_t \cdot (\dots \text{update direction} \dots) 
\end{equation}
\textit{Difference:} LARS/LAMB explicitly scale the learning rate based on parameter and gradient norms. AlphaGrad normalizes the gradient magnitude away first ($\|\tilde{g}_t^L\|_2 \approx 1$) and then shapes the unit direction vector using $\tanh(\alpha \cdot \tilde{g}_t^L)$.

\paragraph{Adam \cite{Kingma2014Adam} / RMSProp \cite{HintonRMSProp}:} Adapt learning rates per-parameter using historical moments (exponential moving averages of the gradient and its square).
\begin{align}
    m_t &= \beta_1 m_{t-1} + (1-\beta_1) g_t \\
    v_t &= \beta_2 v_{t-1} + (1-\beta_2) g_t^2 \\
    \hat{m}_t &= m_t / (1-\beta_1^t), \quad \hat{v}_t = v_t / (1-\beta_2^t) \\
    \theta_{t+1} &= \theta_t - \eta \cdot \frac{\hat{m}_t}{\sqrt{\hat{v}_t} + \epsilon} \label{eq:adam_update} 
\end{align}
\textit{Difference:} Adam/RMSProp are stateful (requiring $m_t, v_t$) and perform per-parameter adaptation based on history. AlphaGrad is stateless (excluding optional momentum) and performs layer-wise adaptation based only on the current gradient via normalization and the $\tanh$ transform, leading to lower computational overhead per step.

\subsection{Theoretical Motivation and Properties of AlphaGrad}

AlphaGrad's mechanism provides several desirable properties stemming directly from its formulation:

\begin{enumerate}
    \item \textbf{Gradient Magnitude Control (Handling Vanishing/Exploding Gradients):} The core operations inherently address gradient scale issues, a well-known challenge in deep network training \cite{Hochreiter2001LSTM, Pascanu2013Clipping}:
    \begin{itemize}
        \item \textit{Normalization:} $\tilde{g}_t^L = g_t^L / (\|g_t^L\|_2 + \epsilon)$. This step is crucial.
            \begin{itemize}
                \item \textbf{Exploding Gradients:} If $\|g_t^L\|_2$ becomes very large, normalization forces $\|\tilde{g}_t^L\|_2 \approx 1$. This acts as an adaptive, automatic rescaling, preventing excessively large updates without requiring a fixed threshold like standard gradient clipping \cite{Pascanu2013Clipping}.
                \item \textbf{Vanishing Gradients:} If $\|g_t^L\|_2$ becomes very small, normalization still ensures $\|\tilde{g}_t^L\|_2 \approx 1$. This amplifies the directional signal relative to the magnitude, ensuring that even small gradients contribute meaningful directional information to the subsequent step, potentially preventing the optimizer from stalling solely due to diminishing gradient scale.
            \end{itemize}
        \item \textit{Smooth Bounding:} $g'^{L}_t = \tanh(\alpha^L \cdot \tilde{g}_t^L)$. The $\tanh$ function ensures that every component of the final transformed gradient $g'^{L}_t$ is strictly bounded within $(-1, 1)$. This provides an absolute bound on the magnitude of individual update components, contributing further to stability.
    \end{itemize}

    \item \textbf{Tunable Interpolation via $\alpha$:} The steepness parameter $\alpha^L$ controls the behavior:
    \begin{equation}
        g'^{L}_t = \tanh(\alpha^L \cdot \tilde{g}_t^L) \approx
        \begin{cases}
            \alpha^L \cdot \tilde{g}_t^L & \text{if } \alpha^L \to 0 \\
            \operatorname{sign}(\tilde{g}_t^L) = \operatorname{sign}(g_t^L) & \text{if } \alpha^L \to \infty
        \end{cases}
    \end{equation}
    This allows AlphaGrad to smoothly transition between using the scaled normalized gradient direction (preserving relative component magnitudes within the normalized vector) and using purely sign-based information (discarding magnitude information entirely, akin to signSGD \cite{Bernstein2018signSGD}). This tuning capability allows adapting the optimizer's aggressiveness and information preservation.

    \item \textbf{Enhanced Sensitivity in Flat Regions:} As discussed previously, when raw gradients $g_t^L$ are small near minima or plateaus, the normalization step preserves the directional signal $\tilde{g}_t^L$. The subsequent $\tanh(\alpha^L \cdot \tilde{g}_t^L)$ transformation, especially with moderate-to-large $\alpha$, can remain highly sensitive to this directional information due to the steep slope of $\tanh$ around zero input. This mechanism might help AlphaGrad maintain progress and escape such regions more effectively than methods whose step size diminishes directly with $\|g_t^L\|$.

    \item \textbf{Computational Efficiency:} AlphaGrad requires only one norm computation and one element-wise $\tanh$ operation per layer beyond the standard gradient calculation. Unlike Adam/RMSProp \cite{Kingma2014Adam, HintonRMSProp}, it does not require storing or updating auxiliary variables like first and second moment estimates (unless momentum is used, which adds one state variable). This makes it computationally lightweight in terms of both memory footprint and floating-point operations per step compared to adaptive methods like Adam.
\end{enumerate}

\subsection{Optimizer Comparison Summary}

Table \ref{tab:optimizer_comparison} provides a comparative summary of AlphaGrad and related optimizers across key characteristics.

\begin{table}[H]
\centering
\caption{Comparison of Optimizer Characteristics}
\label{tab:optimizer_comparison}
\resizebox{\textwidth}{!}{%
\begin{tabular}{@{}lccccc@{}}
\toprule
Feature & AlphaGrad & SGD + Momentum & Adam / RMSProp & signSGD & LARS / LAMB \\
\midrule
\textbf{Adaptation Scope} & Layer-wise & None & Per-parameter & None & Layer-wise \\
\textbf{Adaptation Basis} & Norm + $\tanh(\alpha)$ & None & Historical Moments & Grad Sign & Trust Ratio \\
\textbf{Requires State?} & No (excl. Momentum) & No (excl. Momentum) & Yes & No & No \\
\textbf{Gradient Scale Handling} & Controlled (Norm + $\tanh$) & Requires Clipping & Mitigated (Variance) & Controlled (Sign) & Less Affected (Rate Scaling) \\ 
\textbf{Key Hyperparams} & $\eta$, $\alpha$ (, $\gamma$) & $\eta$ (, $\gamma$) & $\eta$, $\beta_1, \beta_2, \epsilon$ & $\eta$ & $\eta$, Trust Ratio Params \\ 
\textbf{Compute Cost / Step} & Low & Very Low & High & Very Low & Medium \\
\bottomrule
\end{tabular}%
}
\end{table}

\section{Initial Experimental Setup}
\label{sec:experiments}

To evaluate the generalization and stability properties of \alphagrad{} in baseline reinforcement learning tasks, we conduct a series of benchmark experiments using the CleanRL framework \cite{Huang2022CleanRL}. Our goal is to assess AlphaGrad's performance across varied control scenarios and compare it against the widely used Adam optimizer \cite{Kingma2014Adam}.

\subsection{Environments and Policies}

To evaluate the performance and generalization capability of the AlphaGrad optimizer, we conduct experiments across three reinforcement learning environments, each paired with a distinct policy algorithm from the CleanRL suite \cite{Huang2022CleanRL}:

\begin{itemize}
    \item \textbf{HalfCheetah-v5 (PPO)}: A continuous control task from the MuJoCo suite \cite{Todorov2012MuJoCo}, accessed via Gymnasium \cite{Towers2023Gymnasium}, where the agent must learn to coordinate its body segments and run forward efficiently. We use Proximal Policy Optimization (PPO) \cite{Schulman2017PPO}, a popular on-policy algorithm, to examine AlphaGrad's behavior in high-dimensional continuous action spaces.

    \item \textbf{CartPole-v1 (DQN)}: A classic control task with a discrete action space, available in Gymnasium \cite{Towers2023Gymnasium}, where an agent learns to balance a pole on a moving cart. Deep Q-Network (DQN) \cite{Mnih2015DQN}, a value-based off-policy algorithm, is used to test AlphaGrad in low-dimensional, discrete environments.

    \item \textbf{Hopper-v4 (TD3)}: A continuous control environment from MuJoCo \cite{Todorov2012MuJoCo}, via Gymnasium \cite{Towers2023Gymnasium}, where a single-legged agent learns to hop forward while maintaining balance. We employ Twin Delayed Deep Deterministic Policy Gradient (TD3) \cite{Fujimoto2018TD3}, which is specifically designed for stable learning in continuous domains, and allows us to test AlphaGrad's performance under deterministic policy updates.
\end{itemize}

These environments span a broad spectrum of control complexity and algorithmic paradigms such as discrete versus continuous actions, on-policy versus off-policy learning, and low versus high-dimensional state spaces. This diversity enables a comprehensive evaluation of AlphaGrad's stability, adaptability, and convergence behavior across different RL regimes. For continuous action spaces (e.g., HalfCheetah), the output layer represents the mean of a Gaussian policy, whereas for discrete environments (e.g., CartPole), the output layer parameterizes a distribution over actions \cite{Sutton2018RLIntro}. 

\subsection{\texorpdfstring{Directional Saturation and Non-Linear Gradient Behavior of $\alpha$}{Directional Saturation and Non-Linear Gradient Behavior of alpha}}
During test time, we noticed a unique trait that we had not accounted for. While $\alpha$ in AlphaGrad was thought of as a function of the parameter count per layer (e.g., $\alpha = k \cdot \sqrt{d_\ell}$), its true behavior is fundamentally non-linear. Specifically, the interaction between $\alpha$ and the normalized gradient vector $\hat{g} = g / \|g\|_2$ does not result in uniform scaling across components, but rather acts as a directional gating mechanism driven by $\tanh(\alpha \cdot \hat{g}_i)$. Let's examine this behavior and illustrates how $\alpha$ modulates the optimizer's response to gradient heterogeneity.

\paragraph{Alpha as a Directional Filter:}
The update rule in AlphaGrad,
\[
g'_i = \tanh(\alpha \cdot \hat{g}_i),
\]
reveals that $\alpha$ controls the sharpness of the transition from linear behavior to saturation. Importantly, when gradients are heterogeneous---i.e., when some components of $\hat{g}$ are significantly larger than others---the effect of $\alpha$ is no longer uniform across the vector. Larger components saturate, while smaller ones remain in the linear or semi-linear regime of $\tanh$.

\paragraph{Example 1: Balanced Gradients:} Suppose a normalized gradient vector in a layer of dimension $d = 4$ has the form:
\[
\hat{g} = \left[0.5, 0.5, -0.5, -0.5\right].
\]
Applying $\alpha = 2$, we compute:
\[
g'_i = \tanh(2 \cdot \hat{g}_i) = \left[\tanh(1), \tanh(1), \tanh(-1), \tanh(-1)\right] \approx \left[0.7616, 0.7616, -0.7616, -0.7616\right].
\]
This results in a smooth and proportional update across all directions, preserving gradient shape while gently compressing its magnitude.

\paragraph{Example 2: Spiked Gradient Distribution:} Now consider a case with a dominant component:
\[
\hat{g} = \left[0.9, 0.05, 0.03, 0.02\right],
\]
and apply $\alpha = 150$. Then:
\[
g'_i = \tanh(150 \cdot \hat{g}_i) = \left[\tanh(135), \tanh(7.5), \tanh(4.5), \tanh(3)\right] \approx \left[1.0, 0.9999994, 0.9997532, 0.99505\right].
\]
Despite the $\hat{g}_i$ components summing to unit norm, the scaled gradient is effectively saturated in all directions. This produces a near-binary directional update, where the dominant dimension completely overwhelms the others—mirroring SignSGD-like behavior but with soft edges. However, if the smaller components were slightly smaller still (e.g., $\hat{g}_i < 0.01$), they would remain in the linear region of $\tanh$, resulting in a mixed-mode update: one saturated direction and others contributing small, non-saturating nudges.

This behavior demonstrates that $\alpha$ does not uniformly control the optimizer’s aggressiveness across all dimensions. Instead, it introduces gradient-wise selectivity that enables AlphaGrad to resist overcommitment in the presence of spiked or noisy gradient directions, while still preserving signal in weaker components. This allows the optimizer to escape from sharp but narrow minima by maintaining contribution from subtle directions that would otherwise be suppressed. In this way, $\alpha$ serves as a smooth interpolant between full gradient descent and SignSGD-like behavior, not acting merely as a scale factor but functioning as a non-linear filter that sculpts the geometry of descent in a directionally aware and task-adaptive manner.

\subsection{Framework Implementation Differences}
\label{sec:framework_impl}

While AlphaGrad is designed to be conditionally stateless and modular, the behavior of its gradient normalization step can differ depending on the implementation framework due to how layers and parameter tensors are organized.

In PyTorch \cite{Paszke2019PyTorch}, parameters are typically grouped by layers (e.g., \texttt{Linear.weight}, \texttt{Linear.bias}), and it is common practice to structure optimizers using \texttt{parameter\_groups} that reflect this layer-wise organization. As a result, the AlphaGrad implementation in PyTorch computes the L2 norm and applies the \texttt{tanh}-based clipping at the \textit{layer level}, treating each parameter group as a unit. This has the advantage of maintaining intra-layer gradient structure, but may struggle when layers differ significantly in scale, potentially leading to uneven update magnitudes across the network. 

In contrast, a TensorFlow-style \cite{Abadi2016TensorFlow} implementation typically operates over \texttt{model.trainable\_variables}, which are individual tensors rather than grouped layers. Consequently, the same normalization operation—computing $\tilde{g}_t = \nabla f / \| \nabla f \|$ and applying $\tanh(\alpha \cdot \tilde{g}_t)$—is applied independently per tensor. This approach offers finer granularity and may reduce inter-tensor gradient imbalance.

\textbf{Early testing on a private regression dataset (not shown in this paper)} suggests that \textit{tensor-wise normalization} (as done by default in TensorFlow) led to improved convergence stability and smoother loss trajectories compared to layer-wise normalization. However, further controlled experimentation across diverse tasks and architectures is required to determine the optimal granularity of normalization.

This finding motivates future exploration into adaptive normalization schemes that can dynamically balance between tensor-level and group-level gradient control based on model scale or task complexity.

\section{Results and Discussion}
\label{sec:results}

This section critically evaluates the performance of AlphaGrad against the widely-used Adam optimizer \cite{Kingma2014Adam} across three distinct reinforcement learning benchmarks. Our analysis explores how AlphaGrad's unique combination of layer-wise normalization and non-linear transformation interacts with different algorithmic paradigms (DQN \cite{Mnih2015DQN}, TD3 \cite{Fujimoto2018TD3}, PPO \cite{Schulman2017PPO}) and task complexities, revealing a highly context-dependent performance profile.

\subsection{CartPole-v1 (DQN)}

We first examined AlphaGrad on the CartPole-v1 benchmark using Deep Q-Networks (DQN). AlphaGrad was configured with $\alpha$ values derived during test time, ($\alpha \in \{93, 186, 279\}$, interestingly, corresponding to $k=1, 2, 3$ if we derived alpha from $d_{model}$ and not per layer). The results, presented in Figure \ref{fig:cartpole_results}, immediately highlight a challenge: while AlphaGrad agents successfully learned the task—with $\alpha=93$ showing rapid initial progress—their performance was marked by significant instability compared to the smooth convergence trajectory of Adam. Higher values of $\alpha$ exacerbated this issue, leading to slower learning, lower peak performance, and more pronounced oscillations, particularly evident for $\alpha=279$. These performance fluctuations directly correlated with large spikes in the TD loss, suggesting periods where the Q-value updates became highly inaccurate or divergent, a behavior less pronounced under Adam's optimization.

\begin{figure}[H]
    \centering
    \begin{adjustbox}{angle=270}
        \begin{minipage}{\textheight}
            \centering
            \includegraphics[width=\textwidth]{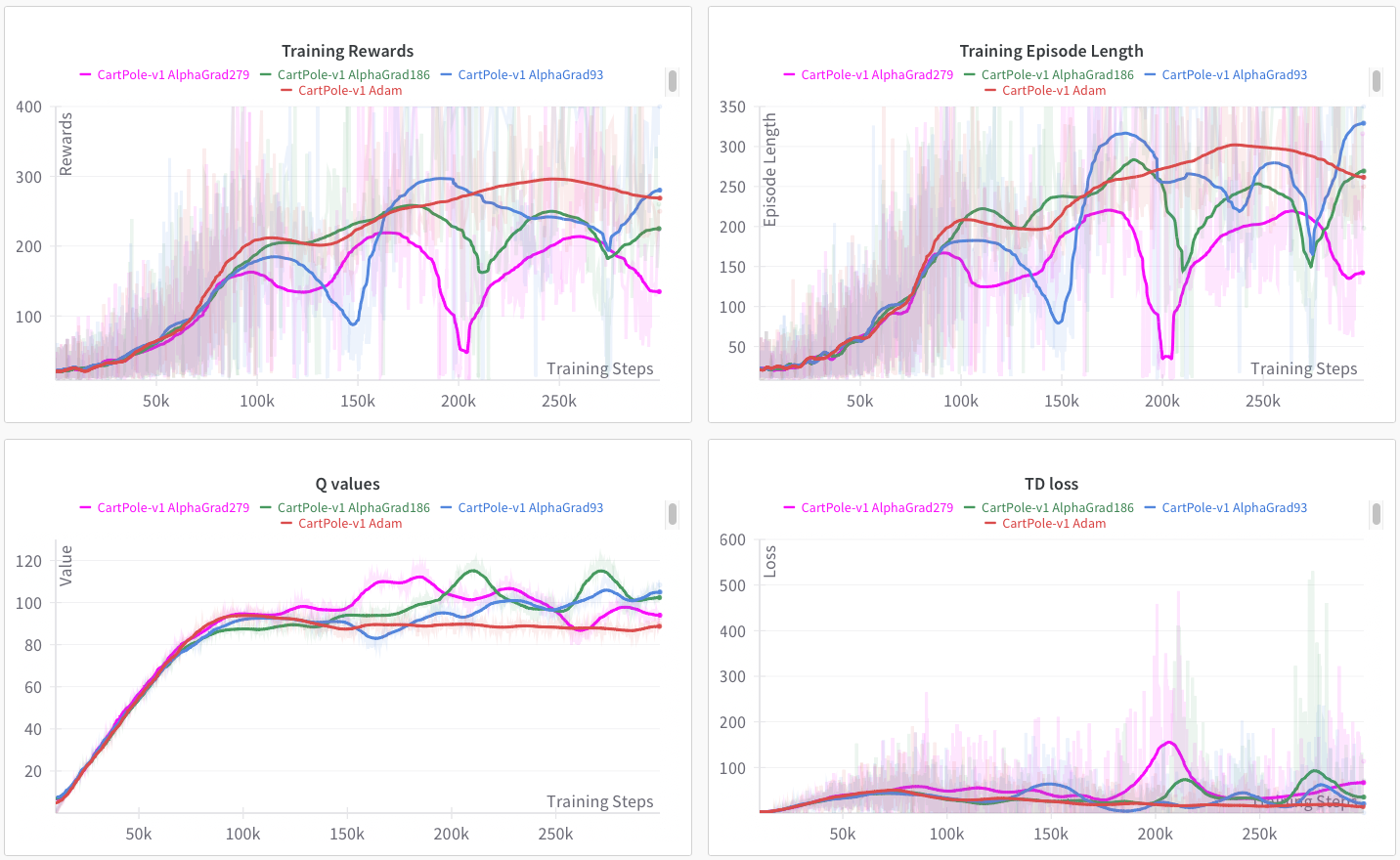} 
            \captionof{figure}{Performance on CartPole-v1 (DQN): Episodic Return, Length, TD Loss, and Q-Values. AlphaGrad variants ($\alpha=93, 186, 279$) vs. Adam.}
            \label{fig:cartpole_results} 
        \end{minipage}
    \end{adjustbox}
\end{figure}

These findings underscore AlphaGrad's sensitivity to $\alpha$ within the DQN framework and suggest a fundamental tension between its core mechanics and off-policy value learning. The relative success of $\alpha=93$ indicates that operating closer to the linear regime of the $\tanh$ function might be less disruptive than the aggressive, near-sign updates induced by larger $\alpha$. The strong saturation associated with high $\alpha$ appears detrimental, likely discarding crucial gradient magnitude information necessary for stable Q-value estimation in this balancing task. This instability likely stems from AlphaGrad's stateless reactivity clashing with DQN's dynamics. DQN relies on potentially noisy target values derived from bootstrapping and experience replay. Adam mitigates this through temporal smoothing via momentum and adaptive scaling via variance estimates \cite{Kingma2014Adam}. AlphaGrad, lacking this history dependence, reacts instantaneously to the current, potentially noisy gradient. While normalization standardizes scale, the subsequent $\tanh$ transform reshapes the update direction based solely on this immediate signal. Coupled with DQN's inherent potential for instabilities (e.g., the "deadly triad" involving off-policy learning, function approximation, and bootstrapping \cite{Sutton2018RLIntro}), this reactivity seems to amplify oscillations rather than suppress them. These results imply that the basic AlphaGrad formulation may require modifications—such as adaptive $\alpha$ scheduling or enhanced momentum strategies—for robust application in off-policy value-based methods, and firmly establish that optimal $\alpha$ necessitates careful empirical tuning beyond simple scaling heuristics.

\subsection{Hopper-v4 (TD3)}

Motivated by the stability challenges observed with DQN, we evaluated AlphaGrad on the more complex Hopper-v4 continuous control task using the Twin Delayed Deep Deterministic Policy Gradient (TD3) algorithm \cite{Fujimoto2018TD3}. Recognizing the inadequacy of the theoretical scaling law (Section \ref{sec:alpha_choice}), we performed an empirical search for effective $\alpha$ values, testing $\alpha \in \{100, 150, 200, 250, 457\}$ against Adam, utilizing a reduced learning rate suitable for the sensitivity of actor-critic methods. The results (Figure \ref{fig:hopper_results}) presented a starkly different outcome.

\begin{figure}[H]
    \centering
    \begin{adjustbox}{angle=270}
        \begin{minipage}{\textheight}
            \centering
            \includegraphics[width=\textwidth]{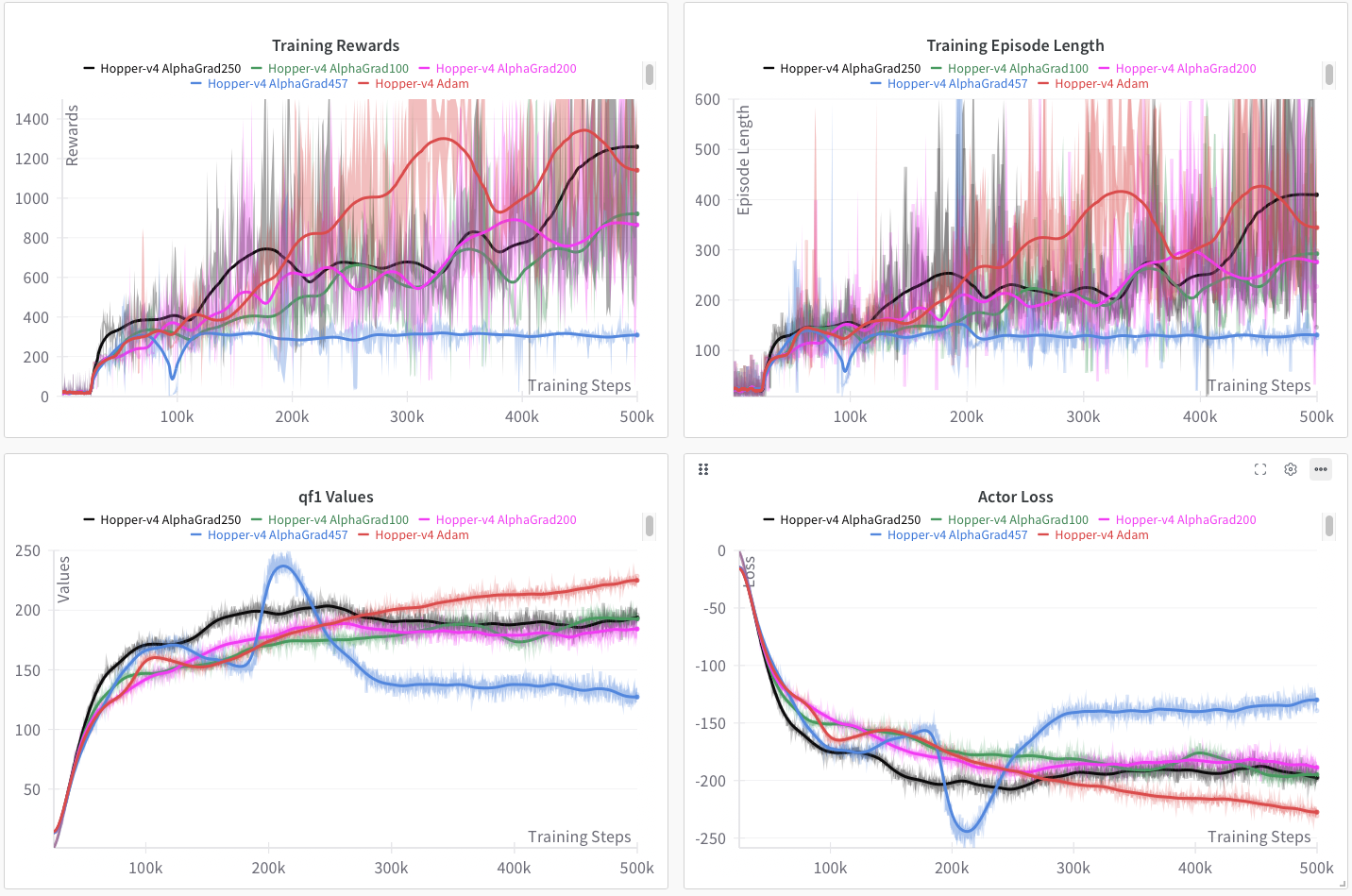} 
            \captionof{figure}{Performance on Hopper-v4 (TD3): Training Rewards, Episode Length, Q-Values, and Actor Loss. AlphaGrad variants (tuned $\alpha$) vs. Adam.}
            \label{fig:hopper_results} 
        \end{minipage}
    \end{adjustbox}
\end{figure}

Confirming the scaling law's limitations for this context, $\alpha=457$ failed to learn effectively. However, empirically identified lower values, particularly $\alpha \in [100, 250]$, yielded strong performance competitive with Adam's peak rewards. The most significant finding was the \textit{superior stability} exhibited by these well-tuned AlphaGrad variants. While Adam achieved slightly higher peak rewards, its learning trajectory showed considerable volatility and significant performance dips. In contrast, AlphaGrad ($\alpha \in [100, 250]$) displayed remarkably smooth, almost monotonic improvement, corroborated by more stable Q-value estimates and actor loss trajectories. This suggests that AlphaGrad's core mechanism—layer-wise normalization ensuring scale consistency, followed by the smooth, bounded `tanh` transformation—when combined with an appropriate learning rate, effectively constrained potentially destabilizing parameter updates that might underlie Adam's volatility in this complex actor-critic setting. The optimal $\alpha$ range ([100-250]) appears to represent a regime prioritizing stability, where the `tanh` function provides sufficient non-linearity and bounding without aggressively discarding magnitude information. This enhanced predictability and robustness could be highly valuable in practical applications like robotics, where consistent improvement and avoidance of performance collapse are often paramount \cite{Kober2013SurveyRLRobotics}. The Hopper results demonstrate that, with careful empirical tuning of $\alpha$ and learning rate, AlphaGrad can match Adam's performance while offering substantially improved training stability, positioning it as a potentially valuable alternative in stability-sensitive domains, possibly benefiting from the inherent stabilization mechanisms within actor-critic frameworks like TD3 \cite{Fujimoto2018TD3}.

\subsection{HalfCheetah-v5 (PPO)}

Our final evaluation focused on the HalfCheetah-v5 environment using the on-policy Proximal Policy Optimization (PPO) algorithm \cite{Schulman2017PPO}. Here, we utilized a larger range of alpha values due to limited compute resources ($\alpha \in \{98, 196, 294\}$) against Adam. The results, shown in Figure \ref{fig:halfcheetah_results}, were striking and demonstrated a clear advantage for AlphaGrad.

\begin{figure}[H]
    \centering
    \begin{adjustbox}{angle=270}
        \begin{minipage}{\textheight}
            \centering
            \includegraphics[width=\textwidth]{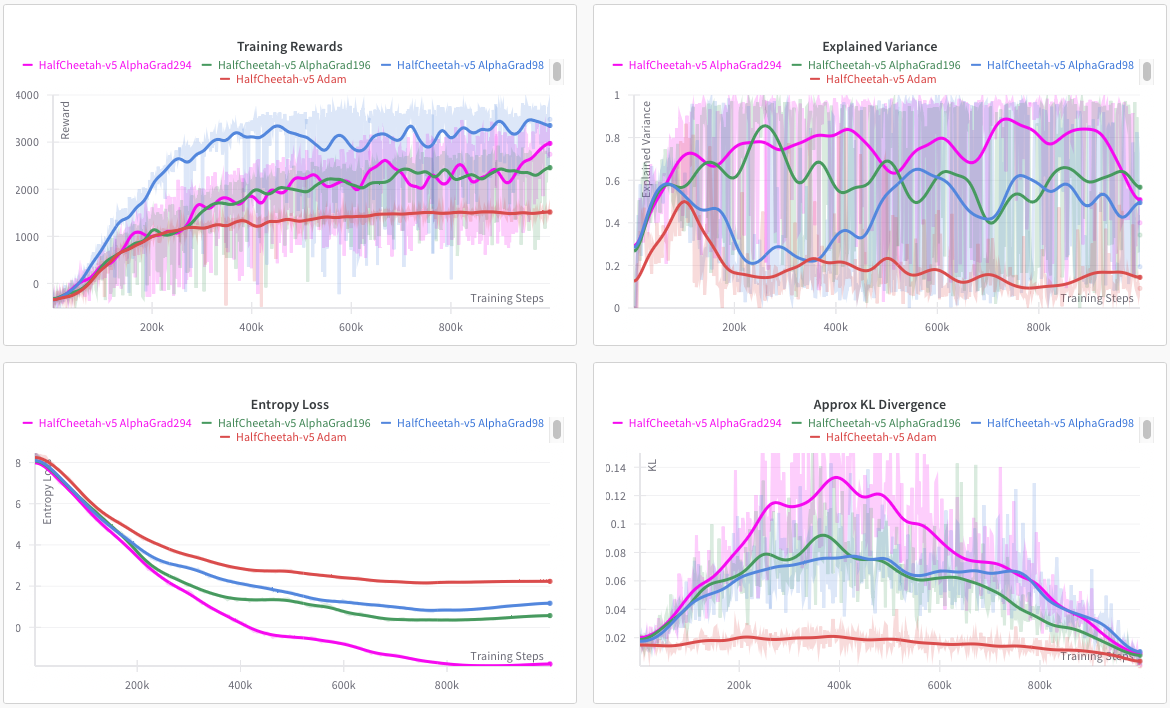} 
            \captionof{figure}{Performance on HalfCheetah-v5 (PPO): Training Rewards, Explained Variance, Entropy Loss, and Approx KL Divergence. AlphaGrad variants ($\alpha=98, 196, 294$) vs. Adam.}
            \label{fig:halfcheetah_results} 
        \end{minipage}
    \end{adjustbox}
\end{figure}

AlphaGrad decisively outperformed Adam in this PPO setting. All AlphaGrad variants achieved substantially higher asymptotic rewards, with $\alpha=98$ peaking above 3000—more than double Adam's plateau around 1500. This performance gain was accompanied by significantly higher explained variance, indicating that AlphaGrad facilitated the learning of a much more accurate value function, crucial for PPO's advantage estimation \cite{Schulman2017PPO}. Furthermore, analysis of policy dynamics revealed faster entropy decay (especially for higher $\alpha$) and markedly larger policy updates per iteration (higher approximate KL divergence) with AlphaGrad compared to Adam. This suggests AlphaGrad enabled PPO to take more aggressive, yet effective, steps within its trust region. The superior performance appears linked to a strong synergy between AlphaGrad's characteristics and PPO's on-policy nature. AlphaGrad's stateless design allows immediate reaction to gradients computed from the latest experience batch, avoiding potential lag from Adam's historical estimates \cite{Kingma2014Adam}. This reactivity, combined with effective gradient rescaling (normalization) and shaping ($\tanh$), seems particularly beneficial for on-policy algorithms that thrive on fresh gradient information. The optimal performance observed at $\alpha=98$ (lowest derived value) hints that even milder non-linearity might suffice, potentially preserving useful magnitude variations while enabling large, effective policy steps.

\subsection{Overall Discussion}

The comparative experiments reveal that AlphaGrad's effectiveness is highly context dependent, shaped by the interplay between its core mechanisms, the RL algorithm's dynamics, and task structure. Its performance relative to Adam varies dramatically: exhibiting instability in off-policy DQN (CartPole), offering superior stability with competitive performance in actor-critic TD3 (Hopper, requiring tuning), and achieving dominant performance in on-policy PPO (HalfCheetah). This sensitivity to the learning algorithm aligns with broader observations about the difficulty of achieving consistent performance across RL benchmarks \cite{Agarwal2021DeepRL}.

A critical outcome is the demonstrated unreliability of the theoretical $\alpha$ scaling heuristic versus the paramount importance of empirical tuning. Optimal performance consistently required searching for $\alpha$, often favoring values deviating significantly from the $\sqrt{d_p}$ guideline derived in Section \ref{sec:alpha_choice}. This highlights that $\alpha$ governs a complex trade-off specific to each learning context.

AlphaGrad's statelessness presents as a double-edged sword. Its immediate gradient reactivity, apparently advantageous in on-policy PPO by enabling rapid adaptation, seemed detrimental in off-policy DQN where it may amplify noise absent Adam's temporal smoothing \cite{Kingma2014Adam}. Yet, the Hopper results suggest that in certain contexts, potentially actor-critic frameworks or scenarios where stability is inherently challenging, AlphaGrad's normalization and bounding can confer superior stability compared to adaptive methods.

In conclusion, AlphaGrad emerges as a promising optimizer with distinct characteristics and a clear potential niche. Its demonstrated strength in on-policy RL warrants significant further investigation. However, its sensitivity to $\alpha$ necessitates robust tuning strategies, and its varied performance profile underscores that optimizer choice remains algorithm- and task-dependent \cite{Ruder2016Overview, Bottou2018OptimizationMethods}. AlphaGrad's value lies in scenarios where its specific benefits—enhanced stability or rapid on-policy adaptation—offer tangible advantages over established adaptive methods.

\section{Further Empirical Validation}
\label{sec:validation_framework}

While preliminary results and theoretical analysis suggest potential benefits for the \alphagrad{} optimizer, a rigorous and comprehensive empirical evaluation across diverse domains is necessary to fully characterize its performance, robustness, and general applicability. Due to the limited budget for this study, testing on standardized benchmarks and LLMs will occur in an addendum. Regardless, this section outlines derivatives of AlphaGrad and motivation that may help discover unknown patterns and characteristics in various domains and modern techniques.

\subsection{Enhancing the Gradient Transformation}
While \alphagrad{} currently relies on the $\tanh$ function to smoothly clip normalized gradients, this choice opens interesting avenues for further exploration. The \textit{shape} and \textit{dynamics} of the non-linearity play a crucial role in how gradient information propagates, especially in regions where gradients are sparse or highly variable. \\

We hypothesize that alternative transformations—particularly those introducing looped or scaled behavior—may further enhance AlphaGrad’s stability and expressivity. For instance, scaling the output of the $\tanh$ non-linearity (\textit{e.g.}, $\beta \cdot \tanh(\alpha x)$) could adjust the effective range of updates, allowing the optimizer to be more responsive without sacrificing the smooth saturation behavior that curtails instability.\\

More intriguingly, we propose testing looped transformations such as:
\[
g' = \sin\left(k \cdot \tanh(\alpha x)\right)
\]
which introduces a softly oscillatory yet bounded transformation. This could help retain directional signal while adding controlled sensitivity in flat regions of the loss surface—particularly relevant in over-parameterized models or late-stage fine-tuning.

These variants are theoretically promising, but their practical value remains to be validated. We consider the systematic benchmarking of such nonlinearities as a promising direction for future research, especially in domains where stability and convergence smoothness are paramount, such as reinforcement learning \cite{Agarwal2021DeepRL}, continual learning \cite{Parisi2019Continual}, or training very deep networks \cite{He2016ResNet}.

\subsection{\texorpdfstring{Layer-wise Adaptive $\alpha^L$ Scheduling}{Layer-wise Adaptive alpha L Scheduling}}
Our current experiments assume a global $\alpha$, shared across all layers. However, layers in modern architectures (e.g., attention vs. feedforward in transformers) exhibit significantly different gradient distributions and dynamics \cite{Xiao2018Dynamical}. A promising direction is to allow $\alpha^L$ to vary per-layer, either:
\begin{itemize}
    \item \textit{Statically:} Initialize $\alpha^L$ as a function of the parameter tensor size $d_p$, using $\alpha^L \propto \sqrt{d_p}$ (as derived in Section \ref{sec:alpha_choice}).
    \item \textit{Dynamically:} Adjust $\alpha^L$ during training based on local statistics such as the saturation ratio $S^L(\alpha)$, potentially adapting to changing gradient characteristics.
\end{itemize}
This would require tuning or learning a small set of $\alpha^L$ values, akin to how learning rates are decoupled across parameter groups in adaptive methods like LARS \cite{You2017LARS} or second-order inspired methods like Shampoo \cite{Gupta2018Shampoo}.

\subsection{Saturation-Aware Scheduling}
We propose a dynamic tuning strategy for $\alpha$ based on the \textit{saturation ratio}, $S(\alpha)$, defined as:
\[
S(\alpha) = \frac{1}{d_p} \sum_{i=1}^{d_p} \mathbb{1}\left[\left|\alpha \cdot \tilde{g}_i\right| > \tau\right],
\]
where $\tau$ (e.g., 2.5 or 3.0) marks the onset of saturation in the $\tanh$ function. A scheduler can then modulate $\alpha$ to maintain $S(\alpha)$ within a target range (e.g., $[0.05, 0.25]$) throughout training.

This saturation-guided strategy could potentially balance early exploratory updates (low $\alpha$) with later-stage convergence stability (high $\alpha$), similar to how cosine or exponential learning rate schedulers influence learning phases by controlling the step size magnitude \cite{Loshchilov2016SGDR}.

\subsection{Applications in Low-Precision and Memory-Constrained Training}
AlphaGrad's statelessness offers practical benefits for deployment in memory-limited or edge scenarios, where the overhead of adaptive optimizers like Adam becomes problematic \cite{Kingma2014Adam}. Because AlphaGrad does not store second moments or large per-parameter state, its memory footprint is significantly smaller. This may be especially beneficial when training quantized networks or TinyML models deployed on resource-constrained hardware \cite{Banbury2021MICROAI}. Furthermore, the bounded nature of updates ($|g'_{t,i}| < 1$) may help avoid overflow or instability in fixed-point arithmetic, which is common in low-precision systems where numerical range is limited \cite{Gupta2015LowPrecision, Jacob2018Quantization}.

\subsection{Integration with Pretraining Regimes and LLM Scaling}
The derived scaling law $\alpha \propto \sqrt{d_p}$ implies that as model size increases, the steepness parameter should grow. This aligns with trends in large model training where gradient scales vary widely across transformer blocks and layers \cite{Kaplan2020Scaling}. However, this behavior introduces concerns of over-saturation if $\alpha$ becomes excessively large.

We propose exploring hybrid update rules where:
\[
g''_t = \lambda \cdot \tanh(\alpha \cdot \tilde{g}_t) + (1 - \lambda) \cdot \tilde{g}_t,
\]
interpolating between the clipped and unclipped normalized update, possibly with adaptive $\lambda$. Such a mechanism could be useful in early pretraining stages or phase transitions between low- and high-precision training, potentially offering more stability than standard optimizers in large-scale settings \cite{Shazeer2018Adafactor, Chen2023Lion}. 

\subsection{Diagnostic Analyses and Broader Benchmarking} 
\label{sec:diagnostics_benchmarking}

To gain deeper insights into AlphaGrad's operational dynamics and its interactions with the optimization process, several diagnostic analyses are warranted. Visualizing the loss landscape geometry navigated by AlphaGrad compared to other optimizers, using techniques like those proposed by Li et al. \cite{Li2018LossLandscape}, could illuminate differences in trajectory curvature, sharpness of minima reached, and update directions. Complementing this, systematic ablation studies isolating the effects of the L2 normalization stage from the subsequent $\tanh$ transformation would precisely quantify the contribution of each component to performance and stability. Furthermore, tracking the distribution of the pre-activation values, $\alpha \cdot \tilde{g}_i$, throughout training could empirically validate whether the gradient components occupy the intended sensitive region of the $\tanh$ function or exhibit excessive saturation or linearity, potentially informing adaptive $\alpha$ strategies. Beyond these diagnostics, rigorous validation necessitates expanding the empirical evaluation to widely adopted benchmarks beyond reinforcement learning. Assessing AlphaGrad's performance on standard supervised learning tasks, such as image classification using MNIST \cite{LeCun1998MNIST}, CIFAR-10/100 \cite{Krizhevsky2009CIFAR}, or ImageNet \cite{Deng2009ImageNet}, would be crucial for characterizing its general applicability and comparing its convergence speed, final accuracy, and generalization properties against established optimizers in these canonical settings.

\section{Final Remarks}
\label{sec:conclusion}

In this work, we introduced AlphaGrad, a memory-efficient optimizer combining layer-wise normalization and a smooth non-linear transformation, and provided initial evidence for its potential utility. Our investigation, constrained by a limited compute budget, focused primarily on reinforcement learning benchmarks. This focus was intentional, driven by the observation that while standard supervised learning benchmarks like MNIST or CIFAR-10 remain valuable, they are increasingly exhibiting saturation effects, where even novel architectural or optimization advances yield diminishing returns against already high baseline performance \cite{Recht2019ImagenetAccuracy, Rajpurkar2018SQuADSaturation}. Furthermore, many real-world applications, particularly those involving sequential decision-making under uncertainty, align more closely with the complexities inherent in RL \cite{DulacArnold2021RLChallenges}. RL environments often feature non-stationarity, high stochasticity, and lack a fixed supervisory signal, making agent behavior highly sensitive to algorithmic components, including the choice of optimizer \cite{Henderson2018DeepRL}. Consequently, RL serves as a challenging stress test, magnifying differences in stability, convergence dynamics, and robustness that might be less apparent in more constrained supervised settings \cite{Agarwal2021DeepRL}. Evaluating optimizers in this regime provides critical insights into their behavior under demanding conditions, which is increasingly relevant as AI systems are deployed in complex, safety-sensitive domains \cite{Amodei2016ConcreteProblemsAISafety}.

Our initial study revealed a compelling, albeit context-dependent, performance profile for AlphaGrad, notably highlighting its capacity for superior stability in complex actor-critic settings (TD3) and significantly enhanced performance in on-policy learning (PPO) compared to Adam, contingent upon careful empirical tuning of the $\alpha$ hyperparameter. While not formally included due to the use of non-standardized datasets, preliminary experiments on text classification tasks using both MLP and Transformer architectures \cite{Devlin2019BERT} also indicated promising convergence behavior. This paper serves as a foundational proposal and analysis of AlphaGrad; extensive future research is required to fully characterize its behavior across diverse domains (including large-scale vision \cite{He2016ResNet} and language modeling \cite{Kaplan2020Scaling}), develop robust hyperparameter tuning strategies (potentially adaptive $\alpha$ methods as discussed in Section \ref{sec:validation_framework}), refine the theoretical understanding (particularly regarding stochastic convergence and dimension dependence, extending Section \ref{sec:convergence}), and definitively establish its place within the broader landscape of deep learning optimization \cite{Bottou2018OptimizationMethods}.

\end{document}